\documentclass[letterpaper, 10 pt, conference]{ieeeconf}  

\IEEEoverridecommandlockouts                              

\usepackage{times}

\usepackage{multicol}
\usepackage{graphicx}
\usepackage{float}
\usepackage{booktabs} 
\usepackage{multirow}
\usepackage{svg}

\usepackage[bookmarks=true]{hyperref}

\usepackage{color}

\usepackage{amsmath, amssymb} 

\usepackage{amsthm}
\newtheorem{theorem}{Theorem}

\usepackage{physics} 
\usepackage{siunitx}

\usepackage{cuted} 
\usepackage{caption}

\title{\LARGE \bf
Sequential Object Placement Optimization with Convex Decomposition
}

\author{Yuezhe Zhang$^{1}$, Xiangyu Lyu$^{1}$, Sohan Rudra$^{1}$, Davide Tateo$^{1, 2, 3}$ and Georgia Chalvatzaki$^{1, 4, 5}$
\thanks{$^{1}$Interactive Robot Perception \& Learning, Technical Univsersity of Darmstadt, Germany; $^{2}$Intelligent Autonomous Systems, Technical Univsersity of Darmstadt, Germany; $^{3}$Robotics and Semantic Systems, Lund University, Sweden; $^{4}$Hessian.AI; $^{5}$Robotics Institute Germany.}%
}

\begin{document}

\maketitle

\setlength{\stripsep}{0pt}
\begin{strip}
  \vspace{-5em}
  \centering
  \includegraphics[width=1.0\textwidth]{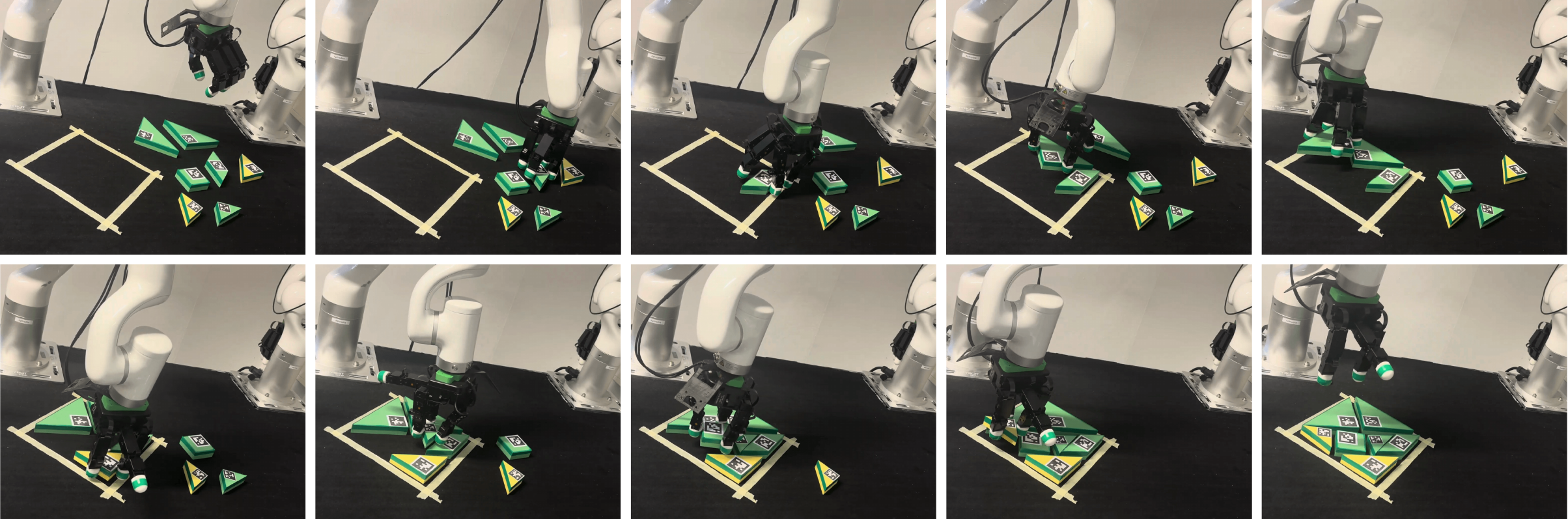}
  \captionof{figure}{A sequence of screenshots of placing the Tangram puzzle using an Allegro Hand and an Xarm.}
  \label{fig:tangram_real}
  \vspace{0.5em}
\end{strip}

\begin{abstract}

Robotic object packing has been a core challenge for robotic deployment in logistics, industry, etc., due to the curse of dimensionality in combinatorial search and the difficulty of dealing with dynamic collision constraints for irregularly shaped objects. Current heuristic and learning-based methods mainly assume a limited spatial discretization resolution of space, and computation becomes extremely inefficient as discretization accuracy increases. In this work, we eliminate this assumption by introducing SOPO-CD, which frames sequential object placement as a differentiable nonlinear optimization problem with hard constraints in a decomposed free space. We formulate the constraints of placing a convex object inside a convex hull as constraining the vertices of the object to lie inside the convex hull. The constraints and their derivatives can be written in closed form and calculated efficiently. We implement a custom solver that achieves local-optimal placements within tightly constrained space in milliseconds; a $50 \times$ speedup compared to a fine-grained grid search method. We evaluate our framework on 2D Tangram, 2D Tetris, and 3D Bin Packing, and have demonstrated strong computational performance and packing utility. We also demonstrate its real-world applicability for solving the Tangram puzzle using a robot equipped with a dexterous hand.





\end{abstract}

\section{Introduction}

Efficiently placing a sequence of objects within a constrained space is a fundamental problem in combinatorial optimization, computational geometry, and robot learning. A central challenge is the continuous nature of object placement: each object must be assigned a position and orientation while satisfying collision and workspace constraints. Many existing approaches \cite{wang2019stable, zhao2021online, funk2022learn2assemble, zhao2023learning, yang2023heuristics} make this problem tractable by discretizing the space of candidate poses. However, finer discretization comes at a rapidly increasing computational cost, while coarse discretization can exclude feasible or tightly packed solutions. Learning-based approaches such as Packing Configuration Trees \cite{zhao2021learning, zhao2025deliberate} instead operate in continuous spaces, but require extensive sampling and exploration, which becomes increasingly difficult for complex and non-convex geometries.

An alternative is to optimize object poses directly in continuous spaces using differentiable geometric constraints \cite{shen2024differentiable, chen2025differentiable}. However, existing approaches typically enforce collision constraints between pairs of geometric primitives, whose number grows as more objects are placed. In addition, they are formulated as soft-constrained optimization problems, making them difficult to converge in a tightly constrained space. In this work, we instead formulate sequential placement as a hard-constrained optimization problem directly in the remaining free space. By decomposing free space into convex hulls, we replace pairwise collision constraints with containment constraints between objects and convex hulls. This allows us to optimize continuous object poses without discretizing the pose space or repeatedly checking collisions against previously placed objects.

We introduce \textbf{Sequential Object Placement Optimization with Convex Decomposition} (\textbf{SOPO-CD}), a framework for sequential object placement through differentiable nonlinear optimization with hard constraints.  SOPO-CD first decomposes the available 2D or 3D space into a set of convex hulls and uses task-specific heuristics to select candidate hulls for placement. For a convex object assigned to a convex hull, the containment can be enforced by constraining its vertices to lie within the hull. This provides an efficient alternative to pairwise collision checking. For non-convex Tetris-like objects, we assign their convex parts to selected convex hulls using heuristics. We derive first- and second-order analytic derivatives of these constraints and the corresponding Lagrangian, and solve the resulting nonlinear optimization problem using a custom Sequential Quadratic Programming (SQP) solver based on \cite{goulart2024clarabel}.




We evaluate SOPO-CD on increasingly complex placement problems spanning 2D Tangram, 2D Tetris, and 3D Bin Packing.
Our analytic derivatives are up to $20\times$ faster than automatic differentiation, while the proposed free-space formulation substantially reduces optimization time compared with differentiable pairwise collision checking. On sequential Tetris and 3D bin packing, SOPO-CD achieves competitive packing utility while providing up to $50\times$ and $900\times$ speedups, respectively, over fine-grained grid-based search. Finally, we demonstrate Tangram placement on a real robotic system equipped with a dexterous hand.

\section{Related Works}

Early interest in the 3D bin packing problem focused primarily on the offline setting, where all items are known a priori and can be placed in arbitrary order. \cite{martello2000three} solved this problem with an exact branch-and-bound approach. However, the problem is strongly NP-Hard, the exact algorithms do not guarantee optimal results within a reasonable amount of time. Therefore, heuristic methods and metaheuristic approaches have been developed to obtain approximate solutions quickly, such as the Bottom-Left (BL) heuristic \cite{baker1980orthogonal} and the Best-Fit-Decreasing heuristic \cite{johnson1974worst}. However, in many real-world application scenarios, e.g., logistics or warehousing, the upcoming items cannot be fully observed. Many online bin packing problems are solved using either heuristics or learning-based methods. Deepest-Bottom-Left-Fill (DBLF) \cite{karabulut2004hybrid} was proposed and combined with a genetic algorithm to place items in the deepest, bottom-most, left-most position. \cite{wang2019stable} performed a grid search on a heightmap and proposed a Heightmap-Minimization method to minimize the volume increase of the packed items. Many RL works \cite{zhao2021online, yang2023heuristics} directly learn their policy on a grid by discretizing the full coordinate space, where the action space grows explosively with the discretization accuracy. Packing Configuration Trees \cite{zhao2021learning, zhao2025deliberate} learns a 3D bin packing policy in a continuous space. However, the packing action space scales with the number of candidate placements, which remains huge in continuous domains, and the sampling and exploration costs are thus high, making it difficult to generalize to non-convex shapes. 


Therefore, we frame the object placement as an optimization problem that handles the continuous space naturally. It is crucial to consider constraints in a differentiable manner. Differentiable collision checking between two convex shapes can be formulated as a convex optimization problem \cite{montaut2023differentiable, tracy2023differentiable} and solved in microseconds. However, the computation time may grow as the number of collision pairs grows. In the packing problem, this is important because the objects to be placed need to check for collisions with all previously placed objects. In addition, it is challenging to do narrow-space planning with non-convex obstacles. We thus take a different view by considering the object placement in the decomposed free space, inspired by \cite{deits2015computing, liu2017planning, tordesillas2021faster, marcucci2023motion}. Some recent works also frame stacking \cite{yang2023compositional, shen2024differentiable, chen2025differentiable} as an optimization problem, but they are soft-constrained and cannot generalize to placing a large number of objects in a tightly constrained space. 


\section{Preliminaries}

\subsection{Problem formulation}


We study the problem of sequential object placement within a constrained space. For a sequence of objects, we only consider placing the current object in a greedy manner. Given the current object set $\mathcal{O}$ and the current collision-free space $\mathcal{C}_{free}$, we aim to find the optimal translation $t \in \mathbb{R}^3$ and rotation parameters $q \in \mathbb{R}^3$ that optimize the objective function $f(t, q)$, such as the depth of the object, and make the object collision-free. 

\vspace{-0.5cm}
\begin{equation}
\label{eq:problem_def}
\begin{split}
\underset{t, q \in \mathbb{R}^3}{\text{minimize}} \quad & f(t, q) \\
\text{s.t.} \quad & g(\mathcal{O}, t, q) \subseteq \mathcal{C}_{free},
\end{split}
\end{equation}

\noindent where $g$ applies the transformations $\{t, q\}$ on the object set $\mathcal{O}$. Note that the object set $\mathcal{O}$ and the collision-free space $\mathcal{C}_{free}$ are not necessarily convex, which makes the problem challenging to solve. 

\subsection{Convex Hull Definition}

The convex hull of a set of points $\mathcal{S} \in \mathbb{R}^d$ is the intersection of all convex sets containing $\mathcal{S}$. We define the vertices of the convex hull as $X = \{x_1, x_2, \dots, x_m\}$. There are two ways of representing the convex hull $conv(X)$.

$\mathcal{V}$-representation writes the convex hull as a convex combination of points in $\mathbb{R}^d$. Mathematically, this reads:
\begin{equation}
\label{eq:convex hull V}
\begin{split}
conv(X) = \biggl\{ x \in \mathbb{R}^d \Big| x = \sum_{i=1}^m \omega_i x_i, \sum_{i=1}^m \omega_i = 1, \omega_i \geq 0 \biggr\},
\end{split}
\end{equation}
\noindent where $\omega_i, i=1, \dots, m$ are the convex coefficients. 

Alternatively, $\mathcal{H}$-representation defines the convex hull as a finite number of inequalities. Mathematically, this yields:
\begin{equation}
\label{eq:convex hull H}
\begin{split}
conv(X) = \biggl\{ x \in \mathbb{R}^d \Big| Ax \leq b\biggr\},
\end{split}
\end{equation}
\noindent where $A$ and $b$ can be computed from $X$ and they determine the half-spaces of the convex hull.

\subsection{Differentiable Collision via LP}

One state-of-the-art differentiable collision checker is proposed by 
\cite{tracy2023differentiable} as solving a Linear Cone Program between a set of convex primitives. It solves for the minimum uniform scaling $\alpha$ that must be applied to the convex primitives $S_1(\alpha)$ and $S_2(\alpha)$ for an intersection to occur. When primitives are not in contact, the minimum scaling $\alpha > 1$, and when the objects are in contact, the minimum scaling $\alpha \leq 1$. The problem is formulated as:
\begin{equation}
    \label{eq:dc}
    \begin{split}
    \underset{\alpha \in \mathbb{R}, x \in \mathbb{R}^3}{\text{minimize}} \quad & \alpha \\
    \text{s.t.} \quad & x \in \mathcal{S}_1(\alpha), \\
    \quad & x \in \mathcal{S}_2(\alpha), \\
    \quad & \alpha \geq 0.
    \end{split}
\end{equation}

When the convex primitives are polytopes, the constraints are half-space constraints of the convex hull.






\begin{figure*}[t]
  \centering
  \includegraphics[width=0.93\textwidth]{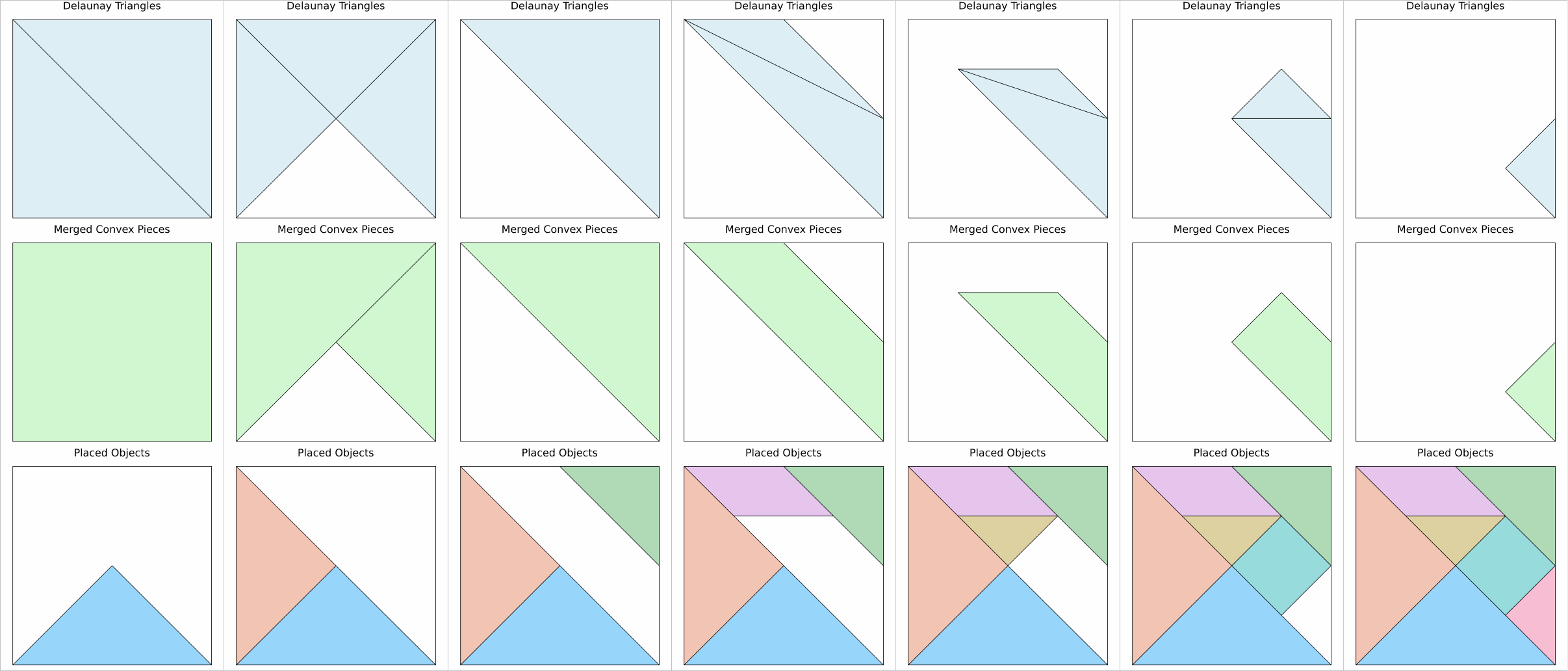}
  \caption{Visualization of the Tangram puzzle with the given object sequence and cost functions. The first two rows show the free-space decomposition using Delaunay Triangulation \cite{VandenHeuvel2024DelaunayTriangulation} and merged convex polygons. At each iteration, we select the largest convex hull, and SQP will calculate the best placement in this convex hull, as shown in the third row.}
  \label{fig:tangram_image}
  \vspace{-1.5em}
\end{figure*}



\section{Methodology}

\subsection{Convex Decomposition}

We decompose the collision-free space into a set of convex hulls $\mathcal{C}_{free} = \mathcal{C}_1 \cup \mathcal{C}_2 \cup \cdots \cup \mathcal{C}_L$ while keeping $L$ as small as possible for every placement.
For 2D problems, we use constrained Delaunay Triangulation \cite{VandenHeuvel2024DelaunayTriangulation} to split the non-convex free space into a number of triangles. It will output a list of valid triangles, and adjacent triangle pairs. We implement a greedy algorithm to merge adjacent triangles into larger convex polygons. We will only merge the triangles/polygons when they share the same edge, and the potential merged polygon is convex, and the area of the merged polygon is large enough. The algorithm runs in a loop and will terminate when there are no adjacent polygons that can be merged. The complexity of the merging algorithm is $O(N^2)$, where $N$ is the number of triangle pieces.



For 3D Packing, the 3D free space is represented as a 2D heightmap of size $M \times M$. We implement a greedy algorithm to partition the free space into a set of axis-aligned Maximal Empty Cuboids \cite{nandy1998maximal}. The algorithm sequentially scans the heightmap to locate uncovered cells. Upon finding a seed cell at height $z$, it executes a two-stage expansion along the $X$ and $Y$ axes to compute a maximal flat bounding region. This region is then extruded vertically to the maximum container height, forming a candidate 3D convex hull. To optimize execution, regions that fail minimum volume or width constraints are logged into a rejection buffer alongside valid hulls. This design ensures that each coordinate is evaluated as a candidate origin at most once, yielding an overall algorithmic complexity of $O(M^2)$.

\subsection{Differentiable Collision Checking in Free Space}

\begin{theorem}
A convex hull $\mathcal{A}$ is inside another convex hull $\mathcal{B}$ if and only if the vertices of $\mathcal{A}$ are inside $\mathcal{B}$.
\end{theorem}

\begin{proof}



Sufficiency: Given that $\mathcal{A} \subseteq \mathcal{B}$ and the vertices of $\mathcal{A}$ are a subset of $\mathcal{A}$, it is obvious that the vertices of $\mathcal{A}$ are a subset of $\mathcal{B}$.

Necessity: Given that the vertices of $\mathcal{A}$ are a subset of $\mathcal{B}$, we want to prove that any point in $\mathcal{A}$ also belongs to $\mathcal{B}$. Since the vertices of $\mathcal{A}$ are a convex combination of the vertices of $\mathcal{B}$, and any point in $\mathcal{A}$ is a convex combination of the vertices of $\mathcal{A}$. It is not difficult to calculate that any point in $\mathcal{A}$ can be expressed as a convex combination of the vertices of $\mathcal{B}$.  







\end{proof}


Each convex hull in the free space can be represented using $A\in \mathbb{R}^{k \times3}$ and $b \in \mathbb{R}^k$ in a world frame $\mathbb{W}$ via $\mathcal{H}$-representation. The number of half-space constraints is $k$. Each object is defined with an attached body reference frame $\mathbb{B}$ with an origin $r \in \mathbb{R}^3$ expressed in a world frame $\mathbb{W}$. The vertices of the object can be expressed as $V_i \in \mathbb{R}^{3}, 1\leq i \leq l$
in frame $\mathbb{B}$, where $l$ is the number of vertices. The translation of the object can be defined as $t\in \mathbb{R}^3$. The orientation of an object is defined by a rotation matrix $^WQ^B \in \mathbb{R}^{3\times3}$ in the world frame, denoted as $Q$ for shorthand.

Therefore, the vertices of the object can be expressed in $\mathbb{W}$ as $r+t+QV_i$. And we can express the constraint that each vertex lies inside the free space as:
\begin{equation}
\label{eq:ineq constraint}
\begin{split}
A(r+t+QV_i) \leq b, 1\leq i \leq l.
\end{split}
\end{equation}

Then the problem of convex object placement in the convex hull can be written as:
\begin{equation}
\label{eq:problem_def2}
\begin{split}
\underset{t, q\in \mathbb{R}^3}{\text{minimize}} \quad & f(t, q) \\
\text{s.t.} \quad & A(r+t+Q(q)V_i) - b \leq 0, 1\leq i \leq l.
\end{split}
\end{equation}

We denote $g_i(t, q) =  A(r+t+Q(q)V_i) - b \in \mathbb{R}^k$. The cost function can be written as:

\begin{equation}
\label{eq:cost function}
\begin{split}
f(t, q) = \omega^T \sum_{i=1}^l (t+Q(q)V_i),
\end{split}
\end{equation}
\noindent where $\omega \in \mathbb{R}^3$ is a given constant that can minimize a linear combination of the $x, y, z$ positions of the sum of all vertices. The problem is nonlinear, as the mapping from the representation $q$ to the rotation matrix $Q$ is nonlinear.

\subsection{Convex Hull Selection and Assignment}

We decompose the free space into a set of convex hulls and formulate object placement constraints accordingly. However, assigning objects to specific hulls leads to a combinatorial search problem. Infeasible assignments—such as placing an object into a region of smaller volume—can be identified immediately. Because exhaustively solving the placement problem across all possible hull assignments is computationally expensive, we introduce heuristic rules to prune the search space, drastically improving computational speed while preserving solution quality.

For the Tangram puzzle, we select the largest free-space convex hull at each iteration. For 3D Bin Packing, candidate hulls are sorted by height and volume, favoring lower and larger hulls first. We then solve the placement optimization for each of the top-$k$ candidate hulls. If no feasible solution exists, the trial fails; otherwise, we select the candidate that optimizes the objective function.

Non-convex object placement is significantly more complex because distinct convex components of a single object may span multiple free-space convex regions. While a conservative approximation—treating the non-convex object by its overall convex hull—suffices in uncluttered settings, it fails in tightly bounded environments such as Tetris-like packing. As shown in Fig. \ref{fig:hull_assignment}, after decomposing the free space into convex regions, we form adjacent convex hull pairs and sort them by height. We then independently solve the placement optimization for the top-$k$ lowest hull pairs. For the convex hulls $\mathcal{B}_m, \mathcal{B}_n$, we denote their half-space constraints as $\bar{A}_m \in \mathbb{R}^{k\times3}, \bar{b}_m \in \mathbb{R}^3$ and $\bar{A}_n \in \mathbb{R}^{k\times3}, \bar{b}_n \in \mathbb{R}^3$. For the non-convex tetrominoes, we can split them into two convex bodies as $\mathcal{A}_a, \mathcal{A}_b$, and we can denote their vertices as $V_i, 1 \leq i \leq k_a$ and $V_i, k_a < i \leq k_b$. Given these two pairs, we can have two assignments $\mathcal{A}_a \subseteq \mathcal{B}_m, \mathcal{A}_b \subseteq \mathcal{B}_n$ or $\mathcal{A}_a \subseteq \mathcal{B}_n, \mathcal{A}_b \subseteq \mathcal{B}_m$. We write the problem for the first assignment as:
\begin{equation}
\label{eq:problem_def_tetris}
\begin{split}
\underset{t, q\in \mathbb{R}^3}{\text{minimize}} \quad & f(t, q) \\
\text{s.t.} \quad & \bar{A}_m(r+t+Q(q)V_i) \leq \bar{b}_m, 1\leq i \leq k_a, \\
                  & \bar{A}_n(r+t+Q(q)V_i) \leq \bar{b}_n, k_a < i \leq k_b.
\end{split}
\end{equation}

\begin{figure}[t]
  \centering
  \includegraphics[width=0.5\textwidth]{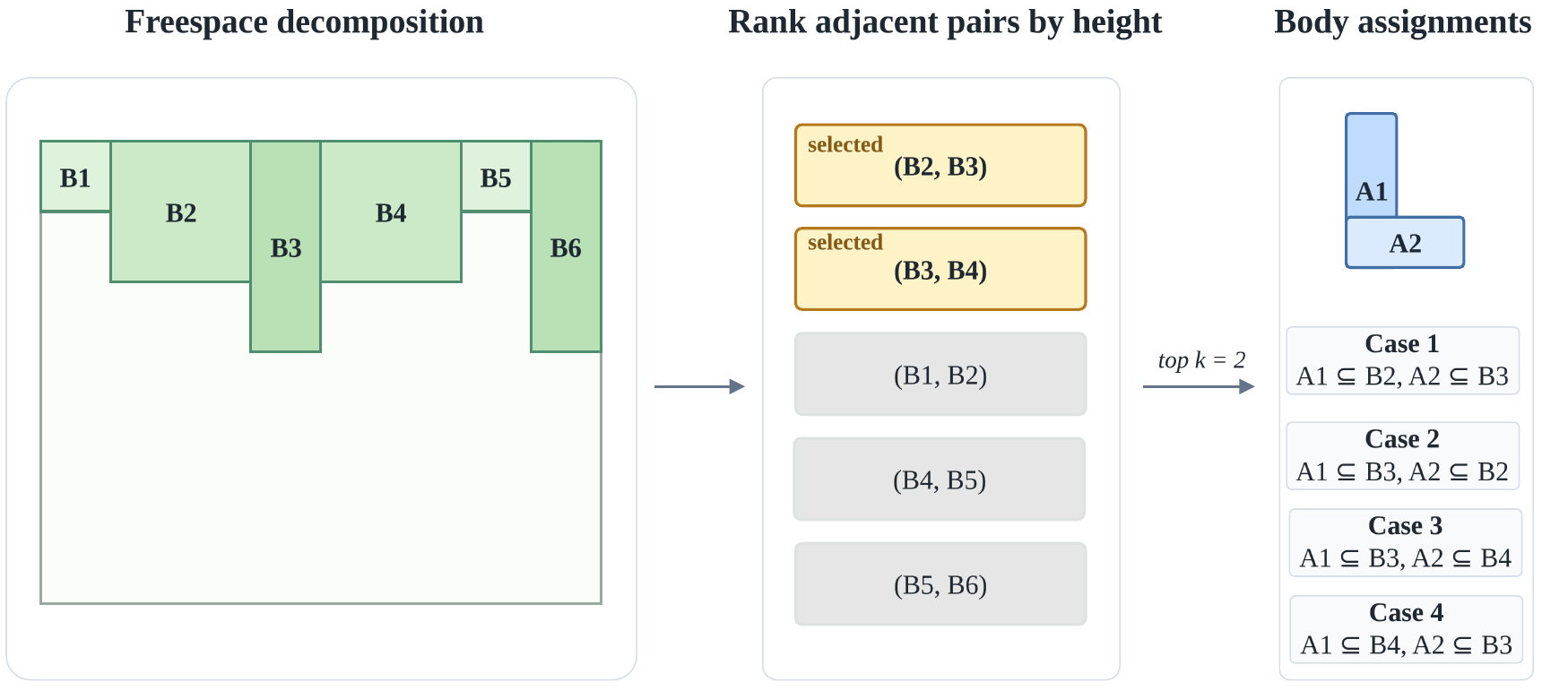}
  \caption{Illustration of hull assignment for placing an L-shaped object. The free-space convex hulls can be grouped into 5 pairs of adjacent hulls. We sort the pairs by height and select the top $k=2$ hull pairs. Each convex body of the object is assigned to each hull, resulting in 4 combinations of constraint pairs, and we will solve these 4 problems separately.}
  \label{fig:hull_assignment}
  \vspace{-2em}
\end{figure}


Here, we can adopt similar objective functions as in \ref{eq:cost function}. If we choose the top-$k$ lowest adjacent hulls, we need to solve $2k$ different optimization problems separately.

\subsection{Compute Analytic Derivative}

We have framed the convex object placement in eq. \ref{eq:problem_def2} and Tetris placement in eq. \ref{eq:problem_def_tetris}. Given that the constraints and cost functions are written explicitly with respect to variables $t, q$, we can calculate the gradient of the cost function, the Jacobian of the constraints, and the Hessian of the Lagrangian function directly in closed form. We denote $x = [t, q] \in \mathbb{R}^6$ and consider the constraint $g_i = A(r+t+Q(q)V_i) - b$. We can write the Jacobians as:
\begin{equation}
\label{eq:analytic gradients}
\begin{split}
\pdv{g_i}{x} 
& = \begin{bmatrix}
\pdv{g_i}{t} & \pdv{g_i}{q}
\end{bmatrix} 
= \begin{bmatrix}
A & A\pdv{Q}{q}V_i
\end{bmatrix} \in \mathbb{R}^{k \times 6},
\\
\end{split}
\end{equation}
\noindent where $\pdv{Q}{q}$ is a third-order tensor and can be explicitly written as $\left( \pdv{Q}{q} \right)_{ijk}
= \pdv{Q_{ij}}{q_k}$. \footnote{$q_k$ means the $k$-th value of vector $q$ and does not mean the covariant component.}

Since $g_i \in \mathbb{R}^k$, the second-order derivative of $g_i$ with respect to $x$ is a third-order tensor. To express it efficiently, we denote $g_{ij} = A_j(r+t+Q(q)V_i) - b_j \in \mathbb{R}, 1 \leq j \leq k$, where $A_j \in \mathbb{R}^{1\times3}$ and $b_j \in \mathbb{R}$ represent each halfspace constraint. Then we can calculate the Hessian as:
\begin{equation}
\label{eq:consrtaint hessian}
\begin{split}
\nabla^2_{xx}{g_{ij}} 
& = \begin{bmatrix}
\nabla^2_{tt}{g_{ij}} & \nabla^2_{tq}{g_{ij}} \\
\nabla^2_{qt}{g_{ij}} & \nabla^2_{qq}{g_{ij}} \\
\end{bmatrix} 
= \begin{bmatrix}
0 & 0 \\
0 & A_j\pdv[2]{Q}{q}V_i \\
\end{bmatrix} \in \mathbb{R}^{6\times6},
\end{split}
\end{equation}
\noindent where $A_j\pdv[2]{Q}{q}V_i \in \mathbb{R}^{3\times3}, \pdv[2]{Q}{q}$ is a fourth-order tensor and can be explicitly written as $\left( \pdv[2]{Q}{q} \right)_{ijkl} = \pdv[2]{Q_{ij}}{q_k}{q_l}$.

We choose the Euler angle as the representation $q = [\alpha, \beta, \gamma]$ and follow the ZYX convention. We have $Q = R_z(\gamma)R_y(\beta)R_x(\alpha)$, where the rotation matrices can be written as:
\begin{equation}
\label{eq:euler rotation}
\begin{split}
R_z(\gamma) &= \begin{bmatrix}
    c_\gamma & -s_\gamma & 0 \\
    s_\gamma & c_\gamma & 0 \\
    0 & 0 & 1 \\
\end{bmatrix}, 
R_y(\beta) = \begin{bmatrix}
    c_\beta & 0 & s_\beta \\
    0 & 1 & 0 \\
    -s_\beta & 0 & c_\beta \\
\end{bmatrix}, \\
R_x(\alpha) &= \begin{bmatrix}
    1 & 0 & 0 \\
    0 & c_\alpha & -s_\alpha \\
    0 & s_\alpha & c_\alpha \\
\end{bmatrix}, 
\end{split}
\end{equation}
where $c$ and $s$ represent $\cos$ and $\sin$ functions. Rotation matrices have some good properties that we can use:
\begin{equation}
\label{eq:rotation derivative}
\begin{split}
\pdv{R_x(\alpha)}{\alpha} &= \begin{bmatrix}
    0 & 0 & 0 \\
    0 & -s_\alpha & -c_\alpha \\
    0 & c_\alpha & -s_\alpha \\
\end{bmatrix} = R_x(\alpha+\pi/2) - E_{11}, \\
\pdv[2]{R_x(\alpha)}{\alpha} &= \begin{bmatrix}
    0 & 0 & 0 \\
    0 & -c_\alpha & s_\alpha \\
    0 & -s_\alpha & -c_\alpha \\
\end{bmatrix} = -R_x(\alpha) + E_{11},
\end{split}
\end{equation}
where $E_{ij}$ is the standard basis matrix that has value 1 at index $i, j$ and 0 elsewhere. Similarly, we can have: 
\begin{equation}
\label{eq:rotation derivative2}
\begin{split}
\pdv{R_y(\beta)}{\beta} & = R_y(\beta+\pi/2) - E_{22},
\pdv[2]{R_y(\beta)}{\beta} = -R_y(\beta) + E_{22}, \\
\pdv{R_z(\gamma)}{\gamma} & = R_z(\gamma+\pi/2) - E_{33},
\pdv[2]{R_z(\gamma)}{\gamma} = -R_z(\gamma) + E_{33}.
\end{split}
\end{equation}

Using the chain rule, we can compute:
\begin{equation}
\label{eq:Q derivative}
\begin{split}
\pdv{Q}{\alpha} & = R_z(\gamma)R_y(\beta)\pdv{R_x(\alpha)}{\alpha} \\
& = R_z(\gamma)R_y(\beta)(R_x(\alpha+\pi/2) - E_{11}), \\
\pdv[2]{Q}{\alpha}{\beta} &= R_z(\gamma)\pdv{R_y(\beta)}{\beta}\pdv{R_x(\alpha)}{\alpha} \\
& = R_z(\gamma)(R_y(\beta+\pi/2) - E_{22})(R_x(\alpha+\pi/2) - E_{11}). \\
\end{split}
\end{equation}

Similarly, we compute $\pdv{Q}{\beta}, \pdv{Q}{\gamma}, \pdv[2]{Q}{\alpha}, \pdv[2]{Q}{\beta}, \pdv[2]{Q}{\gamma}, \pdv[2]{Q}{\alpha}{\gamma}, \pdv[2]{Q}{\beta}{\gamma}$. At every iteration when we have the values of $\alpha, \beta, \gamma$, we first compute the rotation matrices $R_x(\alpha), R_x(\alpha+\pi/2), R_y(\beta), R_y(\beta+\pi/2), R_z(\gamma), R_z(\gamma+\pi/2)$, then we can calculate $\pdv{Q}{q}$ and $\pdv[2]{Q}{q}$ by composing these rotation matrices following eq. \ref{eq:Q derivative}. 

For the cost function $f$ in eq. \ref{eq:cost function}, we can calculate the gradient of the cost function as
\begin{equation}
\label{eq:cost gradient}
\begin{split}
\nabla f  
& = \begin{bmatrix}
    \nabla_t f \\  \nabla_q f
\end{bmatrix} 
= \begin{bmatrix}
    l \cdot \omega \\  \omega^T \pdv{Q}{q} \sum_{i=1}^lV_i
\end{bmatrix} \in \mathbb{R}^6, \\
\end{split}
\end{equation}

Then we can calculate the Hessian of the cost function as:
\begin{equation}
\label{eq:cost hessian}
\begin{split}
\nabla^2_{xx} f &= \begin{bmatrix}
    0 & 0 \\
    0 & \nabla^2_{qq} f
\end{bmatrix} \in \mathbb{R}^{6 \times 6}, \\
\nabla^2_{qq} f &= 
\omega^T \pdv[2]{Q}{q} \sum_{i=1}^lV_i \in \mathbb{R}^{3 \times 3}.
\end{split}
\end{equation}


We use $c(x) \in \mathbb{R}^{kl}$ and $J(x) \in \mathbb{R}^{kl \times 6}$ to denote the whole constraints and Jacobian matrix of the constraints. The Lagrangian function for this problem is
\begin{equation}
\label{eq:my lagrangian}
\begin{split}
\mathcal{L}(x, \lambda) = f(x) + \lambda^Tc(x),
\end{split}
\end{equation}
\noindent where $\lambda \in \mathbb{R}^{kl}$ is the lagrangian multiplier. We can compute the Hessian of the Lagrangian function as:
\begin{equation}
\label{eq:hessian lagrangian}
\begin{split}
\nabla_{xx}^2\mathcal{L}(x, \lambda) &= \begin{bmatrix}
    0 & 0 \\
    0 & \nabla^2_{qq}\mathcal{L}
\end{bmatrix}, \\
\nabla^2_{qq}\mathcal{L}(x, \lambda) & = \nabla^2_{qq}f + \sum_{i=1}^l \sum_{j=1}^k\lambda_{ij} \nabla^2_{qq}{g_{ij}} \in \mathbb{R}^{3\times3},
\end{split}
\end{equation}
%
\noindent where $\nabla_{qq}^2f$ and $\nabla^2_{qq}{g_{ij}}$ can be calculated using eq. \ref{eq:cost hessian} and eq. \ref{eq:consrtaint hessian}. To make the calculations more efficient, we use the Kronecker product and vectorization to eliminate for loops when computing the analytic derivatives. The memory is entirely stack-allocated and one-shot during initialization without tedious forward or backward calculations in automatic differentiation. 




 \subsection{Object Placement Optimization via SQP}

 We solve the nonlinear optimization problem \ref{eq:problem_def2} and \ref{eq:problem_def_tetris} via Sequential Quadratic Programming \cite{nocedal2006numerical}. The idea is to model the optimization problem at the current iterate $(x_n, \lambda_n)$ as a quadratic programming subproblem, then use its solution to obtain a new iterate $(x_{n+1}, \lambda_{n+1})$. Suppose at iterate $(x_n, \lambda_n)$, the quadratic problem is modeled as:
\begin{equation}
\label{eq:sqp}
\begin{split}
\underset{p \in \mathbb{R}^6}{\text{minimize}} \quad & f_n + \nabla f_n^Tp+\frac{1}{2}p^TH p \\
\text{s.t.}\quad & J_np+c_n \leq0,  \\
\quad & \lvert p_i \rvert \leq \rho, i = 4, 5, 6,
\end{split}
\end{equation}
where $J_n$ and $c_n$ are the Jacobian matrix and constraints at $x_n$. Note that in eq. \ref{eq:hessian lagrangian}, $ \nabla^2_{xx} \mathcal{L}_n$ is zero everywhere except the bottom right block $\nabla^2_{qq} \mathcal{L}_n$, and $\nabla^2_{qq} \mathcal{L}_n$ is possibly indefinite. We thus approximate the Hessian as: 
\begin{equation}
\label{eq:sqp_hessian}
\begin{split} 
H = \begin{bmatrix}
    0 & 0 \\
    0 & [\nabla^2_{qq}\mathcal{L}_n]_+
\end{bmatrix} + \sigma I,  \\
\end{split}
\end{equation}
where $[\nabla^2_{qq}\mathcal{L}_n]_+$ is the positive semidefinite part of $\nabla^2_{qq}\mathcal{L}_n$. That is, if $U\Lambda U^T$ is the spectral decomposition of $\nabla^2_{qq}\mathcal{L}_n$, then $[\nabla^2_{qq}\mathcal{L}_n]_+ = U[\Lambda]_+ U^T$ simply zeros out all negative eigenvalues.  We then add a scaled identity matrix $\sigma I$ to guarantee $H$ is positive definite. The radius $\rho$ in the orientation trust regions is important for producing good local approximations. 


We use Clarabel \cite{goulart2024clarabel} to solve the convex QP problem \ref{eq:sqp}. After getting the solution $p_n$, we perform a backtracking line search to calculate how far we should move along the direction $p_n$ and enforce a sufficient decrease in the merit function following the Armijo rule. We choose the merit function $\phi(x, \mu) = f(x) + \mu^T\text{max}\{c(x)-\epsilon, 0\}$, where $\mu \in \mathbb{R}^{kl}$ is the positive penalty parameter that penalizes the constraint violation beyond the tolerance $\epsilon$. The SQP terminates when the norm of the current step $\lVert p_n \rVert_2$ and the constraint violation $\text{max}(c_n)$ are smaller than the tolerance $\epsilon$. We set $\epsilon = \num{1e-4}$ and the maximum SQP iteration number to $50$ for all tasks.

Since the original problems \ref{eq:problem_def2} and \ref{eq:problem_def_tetris} that we are addressing are non-convex, SQP is not guaranteed to find the global solution. In the tight packing scenario, finding a near-global optimum is desirable: a local optimum may have zero constraint violation, yet its failure to achieve the best placement leaves no room for the following objects to be placed. We thus parallelize the SQP solver using randomized initial states and select the final placement based on the lowest merit function value among feasible solutions.

\section{Results}

We implement our algorithms in Julia and leverage BenchmarkTools \cite{chen2016robust} to efficiently run all our modules multiple times to evaluate computational performance. We run experiments on a computer with an AMD Ryzen 9 7950x3D 16-core CPU. 

\subsection{Analytic Derivatives \textup{vs.} AutoDiff}

We first compare the computational performance between analytic derivatives and ForwardDiff \cite{RevelsLubinPapamarkou2016}. During each run, we randomize the object's state in 2D and 3D scenarios and compute the gradient of the cost function, the Jacobian of the constraints, and the Hessian of the Lagrangian. We report the median computation time in Table \ref{tab:analytic_derivative comp}. At every run, the values of the analytic derivatives are identical to the values of ForwardDiff. The computation time of analytic derivatives is consistently better than ForwardDiff in 2D and 3D scenarios, especially for the most time-consuming Hessian calculation; the analytic solution is almost $10\text{--}20$ times faster. 

\begin{table}[tp]
  \centering
  \caption{Computation time for calculating gradient, Jacobian, and Hessian.}
  \label{tab:analytic_derivative comp}
  \resizebox{0.5\textwidth}{!}{
  \begin{tabular}{lcccc}
    \toprule
    & scenario & gradient (\unit{\ns}) & Jacobian (\unit{\ns}) & Hessian (\unit{\ns}) \\
    \midrule
    \multirow{2}{*}{\textbf{Analytic}} 
    & 2D & 16.29 & 24.96 & 41.45  \\
    & 3D  & 85.64  & 188.70  & 368.59   \\
    \cmidrule(lr){1-5}
    \multirow{2}{*}{\textbf{ForwardDiff}} 
    & 2D & 21.16 & 45.70 & 311.65 \\
    & 3D  & 167.82  & 345.27  & 7652.00   \\
    \bottomrule
  \end{tabular}}
  \vspace{-0.3cm}
\end{table}

\subsection{Solve Tangram Puzzle}


\begin{table}[bp]
  \centering
  \caption{Computation time for Tangram on a single thread.}
  \label{tab:tangram_time}
  \resizebox{0.5\textwidth}{!}{
  \begin{tabular}{llccccccc}
    \toprule
    & & obj 1 & obj 2 & obj 3 & obj 4 & obj 5 & obj 6 & obj 7 \\
    \midrule
    \multirow{4}{*}{\textbf{SOPO-CD}} 
    & preprocess (\unit{\us}) & 35.34 & 60.35 & 25.05 & 75.44 & 49.64 & 39.94 & 31.78 \\
    & SQP time (\unit{\ms})   
    & 0.31  & 0.21  & 0.16  & 0.32  & 0.88  & 0.20  & 0.26  \\
    & SQP iters               
    & 9.40 & 4.11  & 5.43  & 7.72  & 11.70 & 2.99  & 5.99  \\
    & solved rate (\%)        
    & 100.00 & 100.00  & 99.89  & 99.89  & 99.18 & 100.00  & 99.99  \\
    
    \cmidrule(lr){1-9}
    \multirow{3}{*}{\textbf{DCOL}} 
    & SQP time (\unit{\ms})   
    & 0.31  & 1.01  & 1.00  & 2.25  & 6.13 & 20.13 & 23.19 \\
    & SQP iters               
    & 9.15 & 7.72 & 6.46  & 10.43 & 27.80 & 49.63 & 50.00 \\
    & solved rate (\%)        
    & 100.00 & 99.16  & 99.73  & 91.93  & 57.85 & 0.81  & 0.00  \\
    \bottomrule
  \end{tabular}}
\end{table}

We evaluate our method on a Tangram puzzle as shown in Fig. \ref{fig:tangram_image}. Here, the object sequence $\mathcal{O}_i$ and corresponding cost function $\omega_i$ are given, and we aim to find the placement at each time step that minimizes the cost function while keeping it collision-free. The container size is $[1, 1]$. We compare SOPO-CD to the same SQP with Differentiable Collision (DCOL) \cite{tracy2023differentiable}, where DCOL uses ForwardDiff to calculate the constraint Jacobians and Hessian of the Lagrangian by default. We report the median computation time, average converged SQP iterations, and average solved rate for placing each object in Table \ref{tab:tangram_time}. For each object, the convex decomposition preprocessing takes $25\text{--}75\ \unit{\us}$. With convex decomposition, SQP shows consistent performance, faster convergence ($0.2\text{--}0.8\ \unit{\ms}$ per object), fewer iterations, and a higher feasibility rate. In contrast, SQP with DCOL suffers from scaling issues: computation time grows significantly with the number of objects due to the expansion of active collision pairs.


\begin{table}[tbp]
  \centering
  \caption{Performance for full sequence Tangram.}
  \label{tab:tangram_time_whole}
  \resizebox{0.5\textwidth}{!}{
  \begin{tabular}{lcccc}
    \toprule
    & n\_threads & total time (\unit{\ms}) & successful placements & memory estimate (MiB) \\
    \midrule
    \multirow{4}{*}{\textbf{SOPO-CD}} 
    & 2 & 9.02   & 4.35 & 2.53  \\
    & 4  & 9.31  & 5.62 & 4.66   \\
    & 8  & 10.02 & 6.63 & 9.44   \\
    & 16  & 12.98 & 6.96 & 17.67   \\
    \cmidrule(lr){1-5}
    \multirow{4}{*}{\textbf{DCOL}} 
    & 2 & 27.48 & 2.63 & 3.26 \\
    & 4  & 46.34  & 3.54  & 6.76   \\
    & 8  & 91.79 & 4.53  & 18.26   \\
    & 16  & 123.58 & 4.95  & 61.95   \\
    \bottomrule
  \end{tabular}}
  \vspace{-0.3cm}
\end{table}

 We also report the median total time, average number of successful placements, and memory estimates for placing the full sequence of objects in Table \ref{tab:tangram_time_whole}. SOPO-CD consistently performs better than SQP with DCOL. With 16 threads, our method places almost all 7 objects successfully in about $10 \unit{\ms}$ of total computation time. We visualize the SQP convergence behavior for single-object placement in Fig.~\ref{fig:sqp_conv}. Each method is evaluated over 100 trials on a single thread. The ground-truth global minimum was computed analytically for tight-packing configurations where optimal geometric placements are known. Empirically, SOPO-CD exhibits significantly faster convergence than DCOL and achieves a higher rate of convergence to the near-global optimum. For placements in highly constrained spaces (e.g., obj 6), SQP with DCOL consistently fails due to persistent infeasibility in the QP subproblems. Given its difficulty in handling hard, non-convex collision constraints in a narrow space, subsequent evaluations primarily benchmark against grid-search methods.
 \begin{figure}[b]
  \centering
  \includegraphics[width=0.4\textwidth]{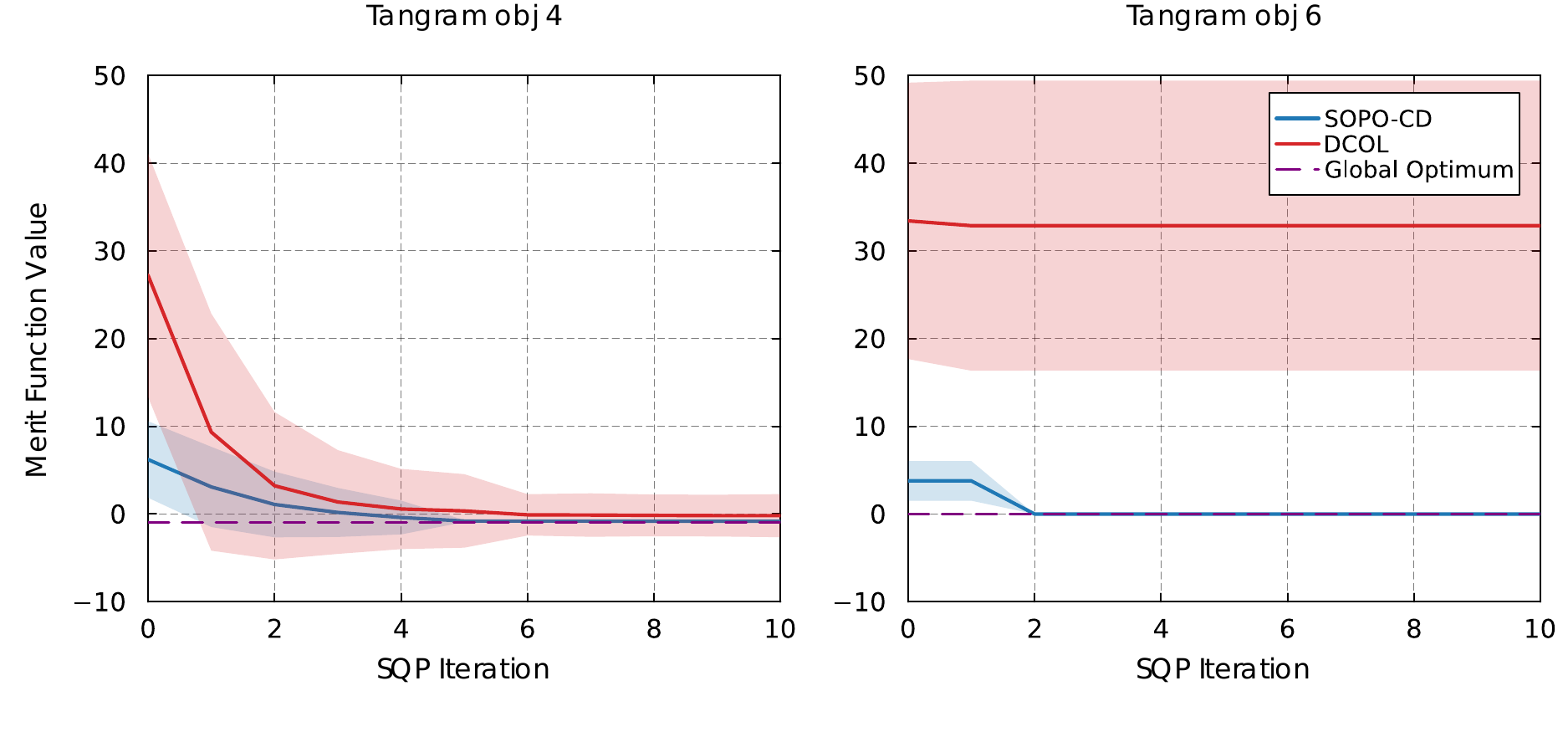}
  \caption{SQP convergence for a single object in Tangram.}
  \label{fig:sqp_conv}
\end{figure}

\subsection{Solve 2D Tetris Puzzle}

We evaluate our method on the 5-Tetris and 8-Tetris benchmarks introduced by cuTAMP \cite{shen2024differentiable} and SPaSM \cite{chen2025differentiable}. The goal is to determine feasible object poses for the 5 and 8 objects so that they fit within a constrained bounding box. The results of our method are shown in Table \ref{tab:tetris_time_whole}. The total time reports the total average time for convex decomposition, constraint extraction, and batch SQP calculations. The total time for 5-Tetris is around $3 \text{--}5\ \unit{\ms}$. The total time for 8-Tetris is around $7 \text{--} 11\ \unit{\ms}$. The success rate and number of successful placements increase with the number of threads, and it is almost always successful when setting the number of threads to 8. The performance is competitive with SPaSM, and we also observe that running SPaSM with a slightly lower cost threshold may lead to infeasible solutions. 

\begin{table}[bp]
  \vspace{-0.25cm}
  \centering
  \caption{Performance for mini Tetris using SOPO-CD.}
  \label{tab:tetris_time_whole}
  \resizebox{0.5\textwidth}{!}{
  \begin{tabular}{lcccc}
    \toprule
    & n\_threads & total time (\unit{\ms}) & successful placements & success rate (\%) \\
    \midrule
    \multirow{3}{*}{\textbf{5-Tetris}} 
    & 2  & 3.14  & 4.75  & 74.95  \\
    & 4  & 3.46  & 4.94  & 93.71   \\
    & 8  & 4.62  & 4.99  & 99.84   \\
    \cmidrule(lr){1-5}
    \multirow{3}{*}{\textbf{8-Tetris}} 
    & 2  & 10.45 & 7.28  & 35.13 \\
    & 4  & 6.90  & 7.80  & 81.05   \\
    & 8  & 8.22  & 7.99  & 98.71   \\
    \bottomrule
  \end{tabular}}
\end{table}

\begin{figure}[tbp]
  \centering
  \includegraphics[width=0.4\textwidth]{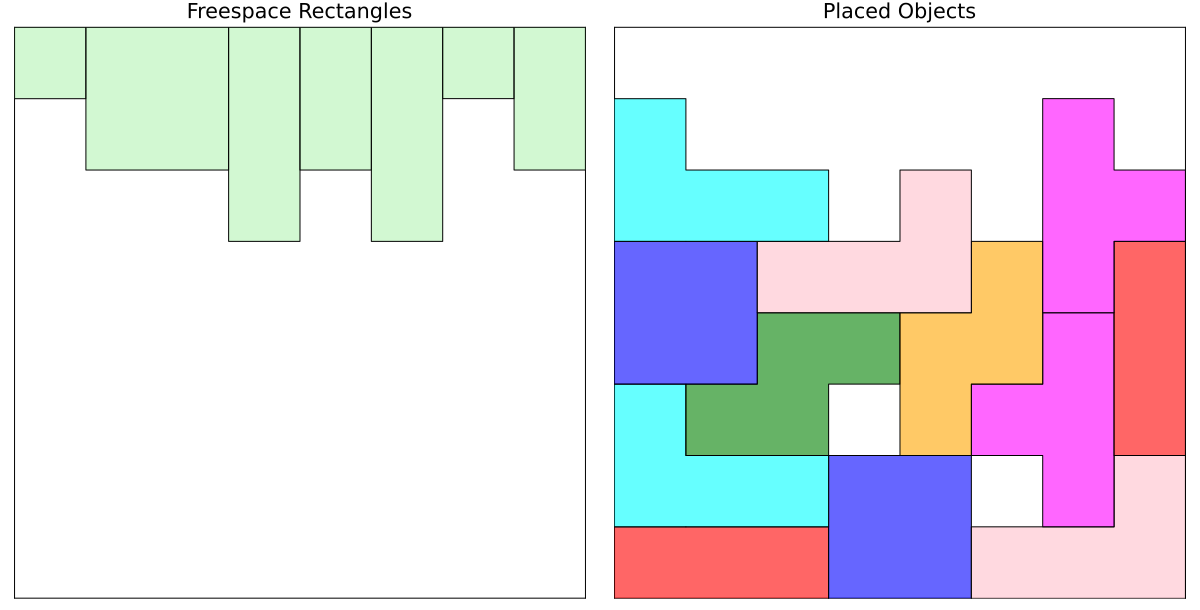}
  \caption{An intermediate screenshot of convex hull generation and placed objects for 2D Tetris.}
  \label{fig:tetris}
  \vspace{-0.5cm}
\end{figure}

We generalize our method to a full Tetris problem, which includes all 7 different tetrominoes, as shown in Fig. \ref{fig:tetris}. The container size is $[8, 8]$. The objective function for each placement is $f(t, q) = \omega^T \sum_{i=1}^l (t+Q(q)V_i)$, where $\omega = [1, 5]^T$ optimizes the object to be at the bottom-left area. We test our method against a grid search algorithm on a number of random sequences. We show the results of grid search on different grid sizes in Table \ref{tab:full_tetris_time_whole_search}. We show the results of SOPO-CD on different top-$k$ convex pairs in Table \ref{tab:full_tetris_time_whole}. At each iteration, the preprocessing takes around $6 \unit{\us}$ and solving $2k$ placement optimization problems for each object takes $10\text{--}20\ \unit{\ms}$ in total. As the solver batch size and top-$k$ increase, the number of successful placements and the occupancy rate rise, reaching 77\% when setting the number of threads to 8 and $k = 1$. The computation time of SOPO-CD is consistently performant, and it is $50$ times faster than grid search at the largest grid size. 

\begin{table}[bp]
  \centering
  \caption{Performance for full Tetris using grid search.}
  \label{tab:full_tetris_time_whole_search}
  \resizebox{0.5\textwidth}{!}{
  \begin{tabular}{lcccc}
    \toprule
    & grid size & compute time per obj (\unit{\ms}) & successful placements & occupancy rate (\%) \\
    \midrule
    \multirow{3}{*}{\textbf{Grid-Search}} 
    & 8$\times$8      &  0.03   &  14.10  &  84.86  \\
    & 80$\times$80    &  1.69   &  13.97  &  83.92  \\
    & 800$\times$800  &  1455.24  &  13.67  & 82.29  \\
    \bottomrule
  \end{tabular}}
\end{table}

\begin{table}[bp]
  \vspace{-0.25cm}
  \centering
  \caption{Performance for full Tetris.}
  \label{tab:full_tetris_time_whole}
  \resizebox{0.5\textwidth}{!}{
  \begin{tabular}{lccccc}
    \toprule
    & top-$k$ & preprocess per obj (\unit{\us}) & compute time per obj (\unit{\ms}) & successful placements & occupancy rate (\%) \\
    \midrule
    \multirow{3}{*}{\begin{tabular}{@{}l@{}}\textbf{SOPO-CD} \\ (n\_threads=2)\end{tabular}} 
    & 1  &  5.40  &  9.05   &  9.95   & 59.46  \\
    & 2  &  5.59  &  13.52  &  10.88  & 65.13  \\
    & 3  &  5.83  &  18.70  &  11.06  & 66.23  \\
    \midrule
    \multirow{3}{*}{\begin{tabular}{@{}l@{}}\textbf{SOPO-CD} \\ (n\_threads=4)\end{tabular}} 
    & 1  &  5.73  &  12.39  &  12.06  & 72.28  \\
    & 2  &  6.09  &  19.06  &  12.28  & 73.62  \\
    & 3  &  5.84  &  26.05  &  12.32  & 73.94  \\
    \midrule
    \multirow{3}{*}{\begin{tabular}{@{}l@{}}\textbf{SOPO-CD} \\ (n\_threads=8)\end{tabular}} 
    & 1  &  5.58  &  16.19  &  12.75  & 76.56  \\
    & 2  &  6.27  &  24.02  &  12.61  & 75.66  \\
    & 3  &  6.39  &  34.31  &  12.79  & 76.79  \\
    \bottomrule
  \end{tabular}}
\end{table}

\subsection{Solve 3D Bin Packing}

We evaluate our method on the 3D Bin Packing problem and compare it with the standard grid search algorithm on the heightmap. The container size is $[5, 5, 5]$, and the packed objects are 3D rectangular shapes with sizes $[3, 2, 1], [2, 2, 2], [1, 2, 2], [1, 1, 2]$. We use a heightmap with grid size $M \times M$ to represent the occupied objects. For every placement, we use the same objective function as the grid search algorithm, $f(t, q) =  t_x+t_y+5t_z$, where $t = [t_x, t_y, t_z]$ is the center of the cuboid. This is the same heuristic as in DBLF \cite{karabulut2004hybrid}. The time complexity for the grid search algorithm is $O(M^2)$.


\begin{figure}[tp]
  \centering
  \includegraphics[width=0.4\textwidth]{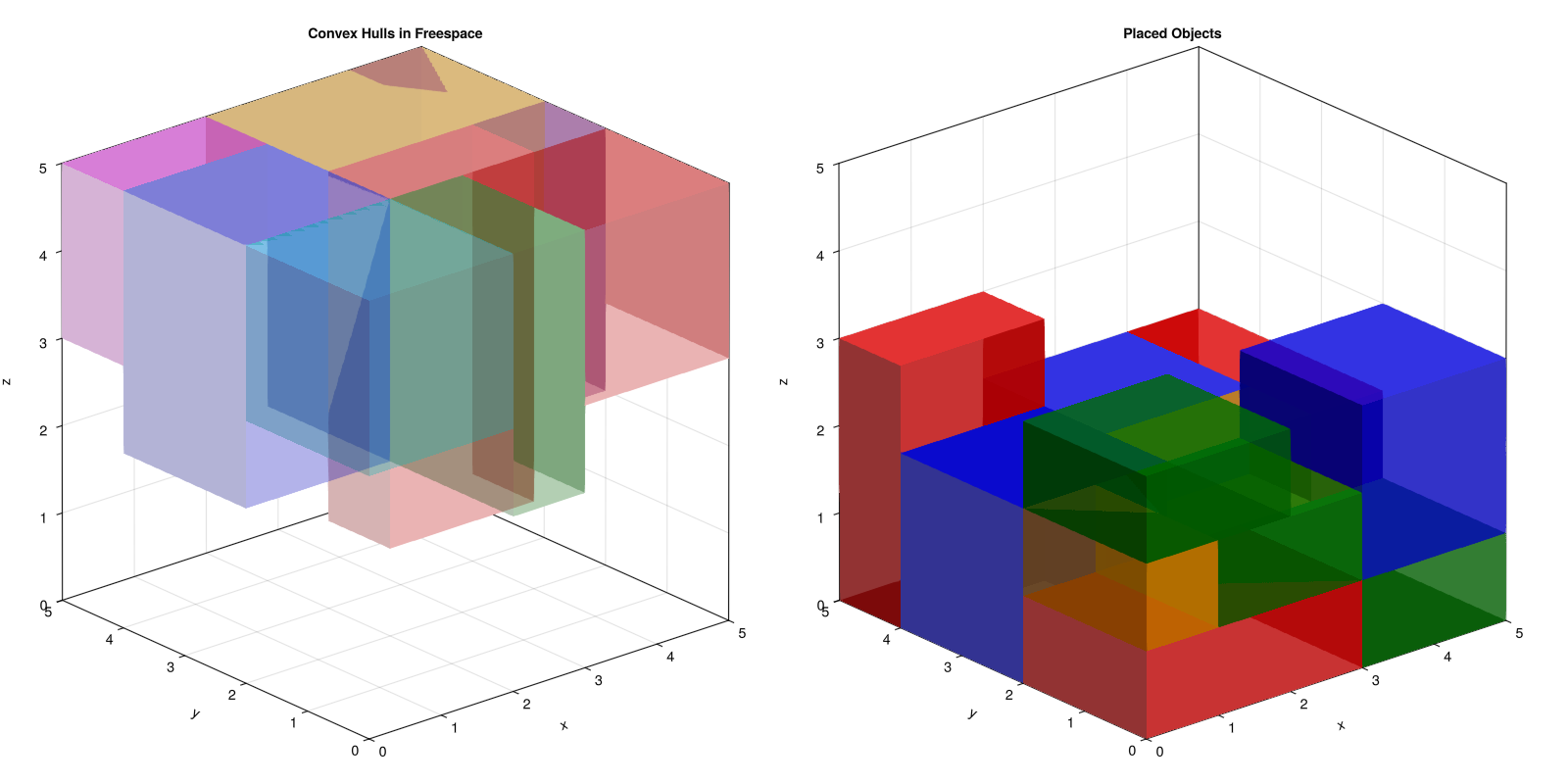}
  \caption{An intermediate screenshot of convex hull generation and placed objects for 3D Bin Packing.}
  \label{fig:packing}
  \vspace{-0.5cm}
\end{figure}



We test SOPO-CD against the grid search algorithm on a number of random sequences of objects. The visualization of an intermediate step is shown in Fig. \ref{fig:packing}. The quantitative results are shown in Table \ref{tab:process_bin_packing} and Table \ref{tab:bin_packing_time_whole}. It is shown that the search time grows drastically as the grid size increases. The preprocessing time for our method also increases with grid size, but remains within 1$\unit{\ms}$. Our method generates a set of axis-aligned convex hulls, and we select the top-$k$ lowest convex hulls and run the batched SQP solver on each hull. The computation time of our method grows linearly with the solver batch size and the number of selected convex hulls $k$. At the largest grid size, our method is more than 900 times faster than the grid search algorithm. The number of successful placements and the occupancy rate also increase significantly as the solver batch size and $k$ increase, and the performance is competitive with that of the grid search. 

\begin{table}[bp]
  \centering
  \caption{Preprocess time of SOPO-CD in 3D.}
  \label{tab:process_bin_packing}
  \resizebox{0.5\textwidth}{!}{
  \begin{tabular}{lccc}
    \toprule
    & \begin{tabular}{@{}c@{}}
    grid size \\ (100$\times$100)\end{tabular}
    & \begin{tabular}{@{}c@{}}
    grid size \\ (250$\times$250)\end{tabular} 
    & \begin{tabular}{@{}c@{}}
    grid size \\ (500$\times$500)\end{tabular}  \\
    \midrule
    Preprocess time per obj (\unit{\us}) &  26.3$\pm$3.29  & 140.4$\pm$15.5  & 496.2$\pm$44.8\\
    \bottomrule
  \end{tabular}}
\end{table}

\begin{table}[tbp]
  \centering
  \caption{Performance for 3D Bin Packing.}
  \label{tab:bin_packing_time_whole}
  \resizebox{0.5\textwidth}{!}{
  \begin{tabular}{lcccc}
    \toprule
    & grid size & compute time per obj (\unit{\ms}) & successful placements & occupancy rate (\%) \\
    \midrule
    \multirow{5}{*}{\textbf{Grid Search}} 
    & 100$\times$100 & 6.28     & 22.04 & 86.63   \\
    & 250$\times$250 & 222.60   & 21.70 & 85.47   \\
    & 500$\times$500 & 3940.55  & 22.20  & 85.97   \\
    \cmidrule(lr){1-5}
    & top-$k$ & compute time per obj (\unit{\ms}) & successful placements & occupancy rate (\%) \\
    \cmidrule(lr){1-5}
    \multirow{3}{*}{\begin{tabular}{@{}l@{}}\textbf{SOPO-CD} \\ (n\_threads=2)\end{tabular}} 
    & 1   & 0.93 & 16.06 & 63.95   \\
    & 2   & 1.67 & 19.31 & 75.54   \\
    & 3   & 2.36 & 20.02 & 78.10   \\
    \cmidrule(lr){1-5}
    \multirow{3}{*}{\begin{tabular}{@{}l@{}}\textbf{SOPO-CD} \\ (n\_threads=4)\end{tabular}} 
    & 1   & 1.22 & 16.68 & 66.52   \\
    & 2   & 2.13 & 20.17 & 78.77   \\
    & 3   & 2.84 & 20.45 & 80.75   \\
    \cmidrule(lr){1-5}
    \multirow{3}{*}{\begin{tabular}{@{}l@{}}\textbf{SOPO-CD} \\ (n\_threads=8)\end{tabular}} 
    & 1   & 1.58 & 17.16 & 68.18  \\
    & 2   & 2.89 & 20.79 & 80.62  \\
    & 3   & 4.05 & 20.52 & 81.60  \\
    \bottomrule
  \end{tabular}}
  \vspace{-0.5cm}
\end{table}

\subsection{Real-world Experiment}

We deploy our algorithm to solve the Tangram puzzle using an Allegro hand and an Xarm in the real world, as shown in Fig. \ref{fig:tangram_real}. We introduce a \SI{2}{\cm} boundary margin beyond the ideal tight layout to account for perceptual noise and execution error. The system operates in an open-loop fashion, passing computed goal poses directly to an off-the-shelf sampling-based motion planner \cite{sucan2012open} for pick-and-place execution. Finger-object contact locations are specified in advance, and grasp configurations are validated using a force-closure optimization program. We define success as placing all blocks inside the boundary. Across 11 experimental runs, 2 trials achieve complete success, and 2 additional trials complete the whole sequence with minor boundary offsets ($\leq$\SI{1}{\cm}). The primary failure modes stem from workspace collisions triggering early stops and visual tracking loss under hand occlusions, highlighting the need for proper contact modeling, force feedback, and reactive planning. All trials are included in the supplementary video.  


\section{Conclusion}

We introduced SOPO-CD, a sequential optimization framework that solves the object placement problem with differentiable hard-constrained collision constraints in free space. We divided the collision-free space into various convex hulls and assigned the convex hull to each convex body. We implemented a custom SQP solver and achieved local-optimal placements in milliseconds. We validated our method on a set of tasks, including 2D Tangram, 2D Tetris, and 3D Bin Packing, and demonstrated strong computational performance and packing utility. We also demonstrated solving a real-world Tangram puzzle using an Allegro Hand and an Xarm. Despite these strengths, some limitations remain to be addressed in the future. The convex hull selection is not exhaustive, and our 3D convex decomposition module generates axis-aligned cuboids; a more general method is needed to handle more complex scenarios. Our algorithm is greedy in choosing the current placement without looking ahead to the future. The objective functions/heuristics for each object are predefined, and it is interesting to combine tree search and backtracking over different combinations of heuristics as in \cite{toussaint2017multi, silver2018general}.

\section*{ACKNOWLEDGMENT}

This project has received funding from the European Union’s Horizon Europe programme under Grant Agreement No. 101120823, project MANiBOT.



\bibliographystyle{IEEEtran}
\bibliography{references}




\end{document}